\documentclass{article}

\usepackage{jmlr2en}

\usepackage{tgtermes}
\usepackage{appendix}
\usepackage{enumerate}
\usepackage{bbm}
\usepackage{tcolorbox}
\usepackage{adjustbox}
\usepackage{bm}
\usepackage{fancyhdr}
\usepackage{pdfpages}\graphicspath{{/Users/mnk/Pictures/}}
\fancyheadoffset[LE,RO]{\marginparsep+\marginparwidth}
\usepackage{amsmath}
\usepackage{amssymb}
\usepackage{mathptmx}      
\usepackage{color}
\usepackage{mathdots}
\usepackage{dsfont}
\usepackage{pdfsync}
\usepackage{enumitem}
\usepackage{mathtools}
\usepackage{subfig}

\newcommand{\enumr}{\begin{enumerate}[label=\roman{*})]}
\newcommand{\enumR}{\begin{enumerate}[label=\Roman{*})]}
\newcommand{\enuma}{\begin{enumerate}[label=\alph{*})]}
\renewcommand{\leq}{\leqslant}
\renewcommand{\le}{\leqslant}  

\renewcommand{\geq}{\geqslant}
\renewcommand{\ge}{\geqslant}

\newcommand{\cc}{\citet}

\newcommand{\argmax}{\mbox{\,\rm arg\,max}}

\newcommand{\eqqrefs}[1]{Eqs.\,\eqref{#1}}
\newcommand{\inqqref}[1]{Ineq.\,\eqref{#1}}

\newcommand{\cpi}{\ensuremath {\pi_{\text{CHK}}}}
\newcommand{\api}{\ensuremath {\pi_{\text{ACF}}}}
\newcommand{\bpi}{\ensuremath {\pi_{\text{BK}}}}
\providecommand{\mceil}[1]{\left \lceil #1 \right \rceil }

\firstpageno{1}

\jmlrheading{1}{2015}{1-48}{4/00}{10/00}{Wesley  Cowan and Michael N. Katehakis }

\ShortHeadings{Normal Bandits of Unknown Means and Variances}{Cowan, Honda and Katehakis}
\firstpageno{1}
\begin{document}
\title{Normal Bandits of Unknown Means and Variances:\\ 
\small Asymptotic Optimality, Finite Horizon Regret Bounds, and a Solution to an Open Problem}
\author{\name Wesley Cowan \email  cwcowan@math.rutgers.edu \\
       \addr   Department of Mathematics\\ Rutgers University\\
110 Frelinghuysen Rd., Piscataway, NJ 08854, USA 
 \AND
       \name Junya Honda  \email honda@it.k.u-tokyo.ac.jp \\
       \addr Department  of Complexity Science and Engineering\\
Graduate School of Frontier Sciences,
The University of Tokyo.\\ 
5-1-5 Kashiwanoha, Kashiwa-shi, Chiba 277-8561, Japan.
       \AND
       \name Michael N. Katehakis  \email mnk@rutgers.edu  \\
       \addr Department of Management Science and Information Systems\\
       Rutgers University\\
 100 Rockafeller Rd., Piscataway, NJ 08854,  USA}


\maketitle
\begin{abstract}
Consider the problem of sampling sequentially from a finite number of $N \geq 2$ populations,
 specified by
random variables  $X^i_k$, $ i = 1,\ldots , N,$ and $k = 1, 2, \ldots$;   where $X^i_k$ denotes
the outcome from population $i$ the $k^{th}$ time it is sampled. It is assumed that for each fixed $i$, 
 $\{ X^i_k \}_{k \geq 1}$ is  a sequence of i.i.d. normal random variables, with unknown mean $\mu_i$ and unknown variance $\sigma_i^2$. 
 The  objective is to have a policy $\pi$  for deciding from which of  the $N$ populations    
 to sample from at any time $t=1,2,\ldots$ so as     
  to maximize the expected sum of outcomes of  $n$ total samples or 
 equivalently to minimize the regret due to lack on information of the parameters  $\mu_i$ and   $\sigma_i^2$.   In this paper, we present a simple inflated sample mean (ISM) index policy that is asymptotically optimal in the 
 sense of Theorem 4 below. This resolves a standing open problem from \cite{bkmab96}. Additionally, finite horizon regret bounds  are given\footnote{Substantial portion of the results reported here were derived independently by Cowan and 
Katehakis,  and  by Honda}.
\end {abstract}
 

\section{Introduction and Summary}

Consider the problem of a controller sampling sequentially from a finite number of $N \geq 2$ populations or `bandits', where the measurements from population $i$ are specified by a sequence of i.i.d. random variables $\{ X^i_k \}_{k \geq 1}$, taken to be normal with finite mean $\mu_i$ and finite variance $\sigma_i^2$.
 The means $\{ \mu_i \}$ and variances $\{ \sigma_i^2 \}$ are taken to be  unknown to the controller. It is convenient to define the maximum mean, $\mu^* = \max_i \{ \mu_i \}$, and the bandit discrepancies $\{ \Delta_i \}$ where $\Delta_i = \mu^* - \mu_i \ge 0 .$ It is additionally convenient to define $\sigma^2_*$ as the minimal variance of any bandit that achieves $\mu^*$, that is $\sigma^2_* = \min_{i: \mu_i = \mu^*} \sigma_i^2$.

In this paper, given $k$ samples from population $i$ we will take the   estimators: 
$\bar{X}^i_k  = \sum_{t = 1}^k X^i_t/k$ and $S_i^2(k) = \sum_{t = 1}^k \left(X^i_t - \bar{X}^i_k\right)^2/k$
for  $\mu_i$ and   $\sigma_i^2$ respectively. Note  that the use of the biased estimator for the variance, with the $1/k$ factor in place of $1/(k-1)$, is largely for aesthetic purposes - the results presented here adapt to the use of the unbiased estimator as well.

For any adaptive, non-anticipatory policy $\pi$, $\pi(t) = i$ indicates that the controller samples bandit $i$ at time $t$. Define $T^i_\pi(n) = \sum_{t = 1}^n \mathbbm{1}\{ \pi(t) = i \}$, denoting the number of times bandit $i$ has been sampled during the periods $t = 1, \ldots, n$ under policy $\pi$;  we take, as a convenience, $T^i_\pi(0) = 0$ for all $i, \pi$.
The  {\sl value} of a policy $\pi$ is the expected sum of the first $n$ outcomes
under $\pi$, which we define to be the  function $V_\pi(n) :$
\begin{equation} \label{eqn:vn}
V_\pi(n)  = 
\mathbb{E}\left[  \sum_{i=1}^N     \sum_{k = 1}^{T^i_\pi(n)}  X^i_k  \right] 
= \sum_{i=1}^N      \mu_i \mathbb{E}\left[ T^i_\pi(n) \right],
\end{equation}
where for simplicity the dependence of  $V_\pi(n)$ on   the true,  unknown,  values of the parameters 
 $\underline{\mu}=(\mu_1, \ldots, \mu_N)$ and  $\underline{\sigma}^2=(\sigma_1^2,\ldots, \sigma_N^2)$, is supressed. 
  The {\sl pseudo-regret}, or simply {\sl regret}, of a policy is taken to be the expected loss due to ignorance of the
  parameters  $\underline{\mu} $ and    $\underline{\sigma}^2$ by the controller. Had the controller complete information, she would at every round activate some bandit $i^*$ such that $\mu_{i^*} = \mu^*= \max_i \{\mu_i\}$. For a given policy $\pi$, we   define the expected regret of that policy at time $n$ as   

\begin{equation} \label{eqn:regret}
R_\pi(n) = n \mu^* -V_\pi(n)
= \sum_{i = 1}^n \Delta_i \mathbb{E}\left[ T^i_\pi(n) \right].
\end{equation}

 It follows from \eqqrefs{eqn:vn} and \eqref{eqn:regret} that 
  maximization of $V_\pi(n)$with respect to $\pi$ is equivalent to minimization of  $R_\pi(n) $. This type of loss due to ignorance of the means (regret) was first introduced  in the 
  context of an $N=2$   problem  by \cite{Rb52} as the `loss per trial' $L_\pi(n)/n =  \mu^* -\sum_{i=1}^N     \sum_{k = 1}^{T^i_\pi(n)}  X^i_k /n$ (for which
  $R_\pi(n)=\mathbb{E}\left[L_\pi(n) \right]$), constructing a modified (along two sparse sequences) `play the winner' policy, $\pi_R$,  such that $L_{\pi_R}(n) =o(n)$ (a.s.) and $R_{\pi_R}(n) = o(n)$, using for his derivation only the assumption of the Strong Law of Large Numbers. 
 Following \cite{bkmab96} when $n\to \infty $, if $\pi$ is 
such that    $R_\pi(n) = o(n)$  we say   policy $\pi$ is     \textbf{\textit{uniformly convergent}} (UC) (since    then 
$\lim_{n\to\infty}  V_\pi(n)/n = \mu^*$ ).  However, if under a policy $\pi$,   $R_\pi(n)$  grew  at a slower pace, 
such as    $R_\pi(n) = o(n^{1/2})$,    or better $R_\pi(n) = o(n^{1/100})$ etc., then 
 the    controller would be  assured that $\pi$   is making a effective  trade-off between exploration and exploitation. It turns our that it  is possible to construct \textbf{\textit{ `uniformly fast convergent' }}  (UFC)  policies, also known as {\sl consistent} or {\sl strongly consistent},  defined as   
  the policies $\pi$ for  which: 
  $$R_\pi(n) = o(n^\alpha), \mbox{ for all $\alpha > 0$ for all $(\underline{\mu}, \underline{\sigma}^2)$} .$$
  
The existence of UFC  policies in the case considered here is well established, e.g.,  \cite{Auer02b} (fig. 4. therein) presented the following UFC   policy $\api$: 
\begin{tcolorbox}[colback=blue!1, arc=3pt, width=.94\linewidth]
\textbf{Policy  $ \boldmath \api $   (UCB1-NORMAL)}. At each $n=1,2,\ldots$:
\begin{itemize}
\item[i)] Sample from any bandit $i$ for which $ T^i_{\api}(n)  < \mceil{8 \ln n}.$
\item[ii)]  If  $ T^i_{\api}(n)  > \mceil{8 \ln n}  $, for all $i=1,\ldots,N,$ sample from bandit $\api(n+1)$ with 
\begin{equation}
\api (n+1) = \argmax_i \left\{  \bar{X}^i_{ T_{\pi}^i (n) } +4 \cdot S_i( T_{\pi}^i (n) ) \sqrt{ \frac{\ln n}{T_{\pi}^i (n)} }   \  \right\}.
\end{equation}
(Taking, in this case, $S^2_i(k)$ as the unbiased estimator.)
\end{itemize}
\end{tcolorbox}

Additionally, \cite{Auer02b}  (in Theorem 4. therein)  gave the following bound:
 
 \begin{equation}\label{en:a-bound}
 R_{\api}(n) \le  M_{\text{ACF}}(\underline{\mu}, \underline{\sigma}^2) \ln n + C_{\text{ACF}}(\underline{\mu}),
\mbox{ for all $n$ and all  $(\underline{\mu}, \underline{\sigma}^2)$} , 
\end{equation}
with
\begin{align}
 M_{\text{ACF}}(\underline{\mu}, \underline{\sigma}^2) &= 256  \sum_{i:\mu_i \neq \mu^*} \frac{\sigma_i^2}{ \Delta_i } + 8 \sum_{i=1}^N \Delta_i , \label{en:am-bound}\\ \label{en:ac-bound}
 C_{\text{ACF}}(\underline{\mu}) &= (1+\frac{\pi^2} {2}) \sum_{i=1}^N \Delta_i.
\end{align}
\inqqref{en:a-bound} readily implies that  $R_{\api}(n) \le  M_{\text{ACF}}(\underline{\mu}, \underline{\sigma}^2) \ln n +o( \ln  n)$.   Thus, since 
$ \ln  n=o(n^\alpha) $   for all $\alpha > 0$ and   $R_{\api}(n) \ge 0 ,$
it follows that  $\api$ is  uniformly fast convergent.

Given that UFC policies exist, the question immediately follows: just how fast can they be? The primary motivation of this paper is the following general result, from \cite{bkmab96}, where they showed that for any UFC policy $\pi$, the following holds:
\begin{equation}\label{en:lower-bound}
\liminf_{n\to\infty} \frac{ R_\pi(n) }{ \ln n } \geq \mathbb{M}_{\text{BK}}(\underline{\mu}, \underline{\sigma}^2), 
\mbox{ for all $(\underline{\mu}, \underline{\sigma}^2)$} , 
\end{equation}
where the bound itself $\mathbb{M}_{\text{BK}}(\underline{\mu}, \underline{\sigma}^2)$  is determined by the specific distributions of the populations, in this case
\begin{equation}\label{en:m-bound}
 \mathbb{M}_{\text{BK}}(\underline{\mu}, \underline{\sigma}^2)
 = \sum_{i:\mu_i \neq \mu^*} \frac{ 2 \Delta_i }{ \ln \left( 1 + \frac{ \Delta_i^2 }{\sigma_i^2} \right)} \ .
\end{equation}

For comparison, depending on the specifics of the bandit distributions, there is a considerable distance between the logarithmic term of the upper bound of Eq. \eqref{en:a-bound} and the lower bound implied by Eq. \eqref{en:m-bound}. 

The  derivation of  \inqqref{en:lower-bound} implies that in order to guarantee  that a policy is 
uniformly fast convergent,  sub-optimal populations  have to be sampled at least a logarithmic number
 of times. The above bound is a special case of a  more general result derived in \cite{bkmab96}
 (part 1 of Theorem 1 therein) 
  for distributions with   multi-parameters being unknown  (such as in the current problem of Normal populations with  both the mean and the variance being unknown): 
     $$\mathbb{M}_{\text{BK}}(\underline{\mu}, \underline{\sigma}^2)
= \sum_{i:\mu_i \neq \mu^*} \frac{ \Delta_i }{\mathbb{K}_i(\underline{\mu}, \underline{\sigma}^2)}$$ 
 with $\mathbb{K}_i(\underline{\mu}, \underline{\sigma}^2)  
 =\inf_{(\mu'_i,\sigma'^2)} \{ \mathbb{I}(f_{(\mu_i,{\sigma}_i^{2})}; f_{(\mu'_i,{\sigma'}_i^2)} ) : {\mu'}_i >\mu^*, {\sigma'}_i^{2}>0 \}
%
 =(1/2) \ln (1+\frac{\Delta_i^2}{\sigma^2_i}).$ 
 
  Previously, 
 \cite{lai85} had obtained such lower bounds  for   distributions  with one-parameter (such as in the current problem of Normal populations with   unknown mean but known variance).  Allocation 
policies that achieved the lower bounds were    called  \textbf{\textit{asymptotically efficient}} 
or   \textbf{\textit{optimal}}   in \cite{lai85}. 
 
\inqqref{en:lower-bound} motivates the definition of a uniformly fast convergent policy $\pi$ as having a \textbf{\textit{uniformly maximal convergence rate}} (UM) or simply being \textbf{\textit{asymptotically optimal}}, within the class of uniformly fast convergent policies,  if  $\lim_{n\to\infty}   R_\pi(n) /\ln n = \mathbb{M}_{\text{BK}}(\underline{\mu}, \underline{\sigma}^2),$ 
since then  $  V_\pi(n) = n \mu^* - \mathbb{M}_{\text{BK}}(\underline{\mu},\underline{\sigma}^2)  \ln n  + o( \ln n ) $. 
 
 \cite{bkmab96}  proposed the following index policy $\bpi$  as one that could achieve this lower bound: 
\begin{tcolorbox}[colback=blue!1, arc=3pt, width=.94\linewidth]
\textbf{Policy $ \boldmath \bpi $   (UCB-NORMAL$\!\,^0$)}
\begin{itemize}
\item[i)] For $n=1,2,\ldots,2N$ sample each bandit  twice, and 
\item[ii)]   for $n \geq 2N$, sample from bandit $\bpi (n+1)$ with 
\begin{equation}
\bpi(n+1) = \argmax_i \left\{ \bar{X}^i_{ T_{\pi}^i (n) } + S_i( T_{\pi}^i (n) ) \sqrt{ n^{\frac{2}{T_{\pi}^i (n)}} - 1}  \ \right\}.
\end{equation}
\end{itemize}
\end{tcolorbox}
\cite{bkmab96} were not able to establish the asymptotic optimality of the 
   $\bpi$  policy because they were not able to establish a sufficient condition ({\sl Condition} A3 therein), which we express here as the following equivalent conjecture (the referenced open question in the subtitle).
\begin{conjecture}
For each $i$, for every $\epsilon > 0$, and for $k \to \infty$, the following is true:
\begin{equation}\label{eqn:prob-conjecture}
\mathbb{P} \left( \bar{X}^i_{ j } + S_i( j ) \sqrt{ k^{2/j} - 1} < \mu_i - \epsilon \text{ for some } 2 \leq j \leq k \right) = o(1/k).
\end{equation} 
\end{conjecture}
We show  that the above conjecture is \textit{false} (cf. Proposition  \ref{prop:conjecture-false} in the Appendix). This does not imply that $\bpi$ fails to be UM (i.e., to be asymptotically optimal), but this failure means that the techniques established in \cite{bkmab96} are insufficient to verify its optimality. All is not lost, however. One of the central results of this paper is to establish that with a small change, the policy $\bpi$ may be modified to  one that is provably asymptotically optimal. We introduce in this paper the policy $\cpi$ defined in the following way:
\begin{tcolorbox}[colback=blue!1, arc=3pt, width=.94\linewidth]
\textbf{Policy     $ \boldmath \cpi $  (UCB-NORMAL$\!\,^2$)}
\begin{itemize}
\item[i)] For $n=1,2,\ldots,3N$ sample each bandit  three times, and 
\item[ii)]   for $n \geq 3N$, sample from bandit $\cpi(n+1)$ with 
\begin{equation}\label{eq.cpi}
\cpi(n+1) = \argmax_i \left\{ \bar{X}^i_{ T_{\pi}^i (n) } + S_i( T_{\pi}^i (n) ) \sqrt{ n^{\frac{2}{T_{\pi}^i (n) - 2}} - 1} \ \right\}.
\end{equation}
\end{itemize}
\end{tcolorbox}

 {\bf Remark 1}

1) Note  that policy   $\cpi$ is only a slight modification of policy  $\bpi$, 
the only difference between their indices   is the $-2$ in the power on $n$ under the radical, i.e.,   $2/(T^i_\pi(n) - 2) $  in  $\cpi(n+1)$
replacing   $2/T^i_\pi(n)$ in $\bpi(n+1)$. This change, while seemingly asymptotically negligible (as in practice $T^i_\pi(n) \to \infty$ (a.s.) with $n$),   has a profound effect on what is provable about $\cpi$.  

2) We note that the indices  of  policy       $\cpi$ 
are a significant  modification of those of the   optimal allocation policy $\pi_{\underline{ \sigma }^2 } $  for the case of normal bandits with {\sl known} variances, cf. \cite{bkmab96}  and \cite{rmab1995},   
  which are: 
$$\pi_{ \underline{\sigma}^2} (n+1) =\argmax_i \left\{ \bar{X}^i_{ T_{\pi}^i (n) } +     \sigma_i \sqrt{ \frac{2\ln n}{T_{\pi}^i (n)} } \  \right \}$$ 
the   difference being replacing the term $ \sigma_i \sqrt{ \frac{2\ln n}{T_\pi^i (n)} } $ in  $\pi_{ \underline{\sigma}^2}$ by 
$S_i( T_\pi^i (n) ) \sqrt{ n^{\frac{2}{T_\pi^i (n) - 2}}} - $  in   $\cpi .$
However, 
 the indices  of  policy       $\api$ are a minor  modification of the  optimal policy $\pi_{\sigma_i} $
 the   difference being replacing the term $ \sigma_i \sqrt{ \frac{2\ln n}{T_\pi^i (n)} } $ in  $\pi_{\sigma_i} $ by 
 $   S_i( T_\pi^i (n) ) \sqrt{ \frac{16 \ln n}{T_\pi^i (n)} } $ in   $\api .$
 
 3) The $\bpi$ and $\pi_{\underline{ \sigma }^2 }$ policies can be seen as connected in the following way, however, observing that $2 \ln n / T_\pi^i(n)$ is a first-order approximation of $n^{2/T_\pi^i(n)}-1 = e^{2 \ln n/T_\pi^i(n)} - 1$.

Following \cite{Rb52}, and additionally  \cite{gittins-79},   \cite{lai85} and \cite{weber1992gittins} there is a large literature on versions of this problem, cf. \cc{burnetas2003asymptotic}, \cc{burnetas1997finite} and references therein.  For  recent work in this area we refer to   \cite{audibert2009exploration}, 
\cite{auer2010ucb}, \cite{gittins2011multi}, \cite{bubeck2012best},  
\cite{cappe2013kullback}, 
\cite{kaufmann14}, 
\cite{2014minimax},  
 \cite{cowan15s}, \cite{cowan2015multi}, 
and references therein. 
    For more general dynamic programming extensions 
we refer to  
\cite{bkmdp97}, \cite{butenko2003cooperative}, \cite{optimistic-mdp}, \cite{audibert2009exploration}, \cite{littman2012inducing}, \cite{feinberg2014convergence} and references therein.     Other related work in this area includes:  \cite{burnetas1993sequencing}, \cite{BKlarge1996}, \cite{lagoudakis2003least},
\cite{bartlett2009regal}, \cite{tekin2012approximately}, \cite{jouini2009multi},
 \cite{dayanik2013asymptotically}, \cite{filippi2010optimism}, \cite{osband2014near}, \cc{dena2013}.

To our knowledge, outside the work in \cite{lai85}, 
  \cc{bkmab96} and \cite{bkmdp97},  asymptotically optimal policies have only been developed in 
 in 
 \cc{honda2011asymptotically}, and  in   \cc{honda2010}  for the   
  problem of finite known support where  
optimal policies, cyclic and randomized, that are simpler to implement than those consider in   \cc{bkmab96} were constructed. Recently in \cc{ck2015u},  
an asymptotically  optimal  policy for uniform bandits of unknown support was 
constructed.  The question of whether  asymptotically optimal  policies   exist in the case discussed herein  of  normal bandits with unknown means and unknown variances was recently resolved in the positive by \cite{honda13} who demonstrated that a form of Thompson sampling with certain priors on $(\underline{\mu}, \underline{\sigma}^2)$ achieves the asymptotic lower bound $\mathbb{M}_{\text{BK}}(\underline{\mu}, \underline{\sigma}^2) .$

The structure of the rest of the paper is as follows.  In section 2, Theorem \ref{thm:finite-time}  establishes a finite horizon bound on the regret of $\cpi$.  From this bound, it follows that $\cpi$ is asymptotically optimal (Theorem \ref{thm:thm-1}), and we provide a bound on the remainder term (Theorem \ref{thm:remainder}). Additionally, in Section \ref{sec:comparison}, the Thompson sampling policy of \cite{honda13} and $\cpi$ are compared and discussed, as both achieve asymptotic optimality.

\section{The Optimality Theorem and Finite Time Bounds}
The main results of this paper, that Conjecture 1 is false (cf. Proposition \ref{prop:conjecture-false} in the Appendix), the asymptotic optimality,  and the bounds on the behavior of $\cpi$, all depend on the following   probability bounds; we note that tighter bounds seem possible, but these are sufficient for this paper.


\begin{proposition}\label{prop:probability-bound}
Let $Z,\ U$ be independent random variables, $Z\sim N(0,1)$ a standard normal, and $U \sim \chi^2_d$ a 
chi-squared distribution with $d$ degrees of freedom, where   $d \geq 2$. 

For $\delta > 0,\ p > 0$, the following holds for all $k \geq 1$:
\begin{equation}
\frac{1}{2}  \mathbb{P} \left( \frac{1}{4} Z^2 \geq U\geq \delta^2 \right)k^{-d/p} \leq \mathbb{P} \left( \delta + \sqrt{ U } \sqrt{ k^{2/p} - 1} < Z  \right) \leq \frac{ e^{-(1+\delta^2)/2} p }{ 2\delta^2 \sqrt{d}  } \frac{ k^{ (1-d)/p } }{ \ln k }.
\end{equation}
\end{proposition}
\begin{proof}[of Proposition \ref{prop:probability-bound}] The proof is given in the Appendix. \end{proof}

\begin{theorem}\label{thm:finite-time}
For  policy $\pi_{\text{CHK}}$ as defined above, the following bounds hold for all $n \geq 3N$ and all $\epsilon\in (0,1)$: 
\begin{equation}\label{en:finite-bound}
R_{\cpi}(n) \leq \sum_{i : \mu_i \neq \mu^* } \left( \frac{ 2 \ln n }{   \ln \left( 1 + \frac{ \Delta_i^2 }{ \sigma_i^2 } \frac{ \left( 1 - \epsilon \right)^2 }{ \left(1 + \epsilon \right) } \right) } + \sqrt{ \frac{ \pi }{ 2 e } } \frac{ 8 \sigma_*^3 }{ \Delta_i^3 \epsilon^3 } \ln \ln n  + \frac{8}{ \epsilon^2} + \frac{ 8 \sigma_i^2  }{ \Delta_i^2 \epsilon^2 } + 4 \right)\Delta_i.
\end{equation}
\end{theorem}

Before giving the proof of this bound, we present two results, the first demonstrating the asymptotic optimality of $\cpi$, the second giving an $\epsilon$-free version of the above bound, which gives a bound on the sub-logarithmic remainder term. It is worth noting the following.  The bounds of Theorem \ref{thm:finite-time} can actually be improved, through the use of a modified version of Proposition \ref{prop:probability-bound}, to eliminate the $\ln \ln n$ dependence, so the only dependence on $n$ is through the initial $\ln n$ term. The cost of this, however, is a dependence on a larger power of $1/\epsilon$. The particular form of the bound given in Eq. \eqref{en:finite-bound} was chosen to simplify the following two results, cf. Remark 4 in the proof of Propositition \ref{prop:probability-bound}.

\begin{theorem}\label{thm:thm-1}
For a policy $\pi_{\text{CHK}}$ as defined above, $\pi_{\text{CHK}}$ is asymptotically optimal in the sense that 
\begin{equation}
\lim_{n\to\infty} \frac{ R_{\pi_{\text{CHK}}}(n) }{ \ln n } = \mathbb{M}_{\text{BK}}(\underline{\mu}, \underline{\sigma}^2).
\end{equation}
\end{theorem}

\begin{proof}[of Theorem \ref{thm:thm-1}]
For any $\epsilon$ such that $0 < \epsilon <1$, we have from Theorem \ref{thm:finite-time} that the followings holds:
\begin{equation}
\limsup_{n\to\infty} \frac{ R_{\pi_{\text{CHK}}}(n) }{ \ln n } \leq  \sum_{i : \mu_i \neq \mu^* } \frac{ 2 \Delta_i }{   \ln \left( 1 + \frac{ \Delta_i^2 }{ \sigma_i^2 } \frac{ \left( 1 - \epsilon \right)^2 }{ \left(1 + \epsilon \right) } \right) }.
\end{equation}
Taking the infimum over all such $\epsilon$, 
\begin{equation}
\limsup_{n\to\infty} \frac{ R_{\pi_{\text{CHK}}}(n) }{ \ln n } \leq  \sum_{i : \mu_i \neq \mu^* } \frac{ 2 \Delta_i }{   \ln \left( 1 + \frac{ \Delta_i^2 }{ \sigma_i^2 } \right) } = \mathbb{M}_{\text{BK}}(\underline{\mu}, \underline{\sigma}^2),
\end{equation}
and observing the lower bound of Eq. \eqref{en:lower-bound} completes the result.
\end{proof}

\begin{theorem}\label{thm:remainder}
For a policy $\pi_{\text{CHK}}$ as defined above, $R_{\cpi}(n) \leq \mathbb{M}_{\text{BK}}(\underline{\mu}, \underline{\sigma}^2) \ln n + O( (\ln n)^{3/4} \ln \ln n )$, and more concretely
\begin{equation}\label{en:ck-bound}
\begin{split}
R_{\cpi}(n) \leq M^0_{\text{CHK}}(\underline{\mu}, \underline{\sigma}^2) \ln n + & M^1_{\text{CHK}}(\underline{\mu}, \underline{\sigma}^2) (\ln n)^{3/4} \ln \ln n  \\
+ &  M^2_{\text{CHK}}(\underline{\mu}, \underline{\sigma}^2) (\ln n)^{3/4} \\ 
+ &  M^3_{\text{CHK}}(\underline{\mu}, \underline{\sigma}^2) (\ln n)^{1/2}  \\
+ &  M^4_{\text{CHK}}(\underline{\mu}, \underline{\sigma}^2),
\end{split}
\end{equation}
where
\begin{equation}
\begin{split}
M^0_{\text{CHK}}(\underline{\mu}, \underline{\sigma}^2) & = \mathbb{M}_{\text{BK}}(\underline{\mu}, \underline{\sigma}^2) \\
M^1_{\text{CHK}}(\underline{\mu}, \underline{\sigma}^2) & = 64  \sqrt{ \frac{ \pi }{ 2 e } } \sum_{i : \mu_i \neq \mu^* } \left(  \frac{\sigma_*^3 }{ \Delta_i^2 }   \right) \\
M^2_{\text{CHK}}(\underline{\mu}, \underline{\sigma}^2) & = 10 \sum_{i : \mu_i \neq \mu^* }\left( \frac{ \Delta_i^3 }{ \left(\sigma_i^2 + \Delta_i^2\right) {\ln\left(1 + \frac{\Delta_i^2}{\sigma_i^2}\right)}^2 } \right) \\
M^3_{\text{CHK}}(\underline{\mu}, \underline{\sigma}^2) & = 32 \sum_{i : \mu_i \neq \mu^* } \left( \Delta_i + \frac{ \sigma_i^2  }{ \Delta_i } \right) \\
M^4_{\text{CHK}}(\underline{\mu}, \underline{\sigma}^2) & = 4 \sum_{i : \mu_i \neq \mu^* } \Delta_i .
\end{split}
\end{equation}

\end{theorem}
While the above bound admittedly has a more complex form than such a bound as in Eq. \eqref{en:a-bound}, it demonstrates the asymptotic optimality of the dominating term, and bounds the sub-linear remainder term.

\begin{proof}[of Theorem \ref{thm:remainder}]
The bound follows directly from Theorem \ref{thm:finite-time}, taking $\epsilon = \frac{1}{2}(\ln n)^{-1/4}$ for $n \geq 3$, and observing the following bound, that for $\epsilon$ such that $0 < \epsilon < 1/2$,
\begin{equation}
\frac{1}{\ln \left( 1 + \frac{ \Delta_i^2 }{ \sigma_i^2} \frac{(1-\epsilon)^2}{1 + \epsilon} \right) } \leq \frac{1}{\ln \left( 1 + \frac{ \Delta_i^2 }{ \sigma_i^2} \right) } + \frac{10 \Delta_i^2}{ \left(\sigma_i^2 + \Delta_i^2\right) \ln\left(1 + \frac{ \Delta_i^2 }{ \sigma_i^2 }\right)^2 } \epsilon.
\end{equation}
This inequality is proven separately as Proposition \ref{prop:that-d-inequality} in the Appendix.
\end{proof}

We make no claim that the results of Theorems \ref{thm:finite-time}, \ref{thm:remainder} are the best achievable for this policy $\cpi$. At several points in the proofs, choices of convenience were made in the bounding of terms, and different techniques may yield tighter bounds still. But they are sufficient to demonstrate the asymptotic optimality of $\cpi$, and give useful bounds on the growth of $R_{\cpi}(n)$.

\begin{proof}[of Theorem 1]
In this proof, we take $\pi = \pi_\text{CHK}$ as defined above. For notational convenience, we define the index function
\begin{equation}
u_i(k, j) = \bar{X}^i_{ j } + S_i( j ) \sqrt{ k^{\frac{2}{j-2}} - 1}.
\end{equation}

The structure of this proof will be to bound the expected value of $T^i_\pi(n)$ for all sub-optimal bandits $i$, and use this to bound the regret $R_\pi(n)$. The basic techniques follow those in 
 \cite{rmab1995} for the known variance case, modified accordingly here for the unknown variance case and assisted by the probability bound of Proposition \ref{prop:probability-bound}. For any $i$ such that $\mu_i \neq \mu^*$, we define the following quantities: Let $1 > \epsilon > 0$ and define $\tilde{\epsilon} = \Delta_i \epsilon/2$. For $n \geq 3N$,
\begin{equation}
\begin{split}
n^i_1(n, \epsilon) & = \sum_{t = 3N}^n \mathbbm{1} \{ \pi(t+1) = i, u_i(t, T^i_\pi(t)) \geq \mu^* - \tilde{\epsilon}, \bar{X}^i_{ T^i_\pi(t) } \leq \mu_i + \tilde{\epsilon} , S^2_i(T^i_\pi(t))  \leq \sigma_i^2(1 + \epsilon) \} \\
n^i_2(n, \epsilon) & = \sum_{t = 3N}^n \mathbbm{1} \{ \pi(t+1) = i, u_i(t, T^i_\pi(t)) \geq \mu^* - \tilde{\epsilon} , \bar{X}^i_{ T^i_\pi(t) } \leq \mu_i + \tilde{\epsilon} , S^2_i(T^i_\pi(t))   > \sigma_i^2( 1 + \epsilon ) \} \\
n^i_3(n, \epsilon) & = \sum_{t = 3N}^n \mathbbm{1} \{ \pi(t+1) = i, u_i(t, T^i_\pi(t)) \geq \mu^* - \tilde{\epsilon} , \bar{X}^i_{ T^i_\pi(t) } > \mu_i + \tilde{\epsilon}  \} \\
n^i_4(n, \epsilon) & = \sum_{t = 3N}^n \mathbbm{1} \{ \pi(t+1) = i, u_i(t, T^i_\pi(t)) < \mu^* - \tilde{\epsilon} \}.
\end{split}
\end{equation}
Hence, we have the following relationship for $n \geq 3N$, that
\begin{equation}\label{eqn:t-n-relation}
T^i_\pi(n+1) = 3 + \sum_{t = 3N}^n \mathbbm{1}  \{ \pi(t+1) = i \} = 3 + n^i_1(n, \epsilon) + n^i_2(n, \epsilon) + n^i_3(n, \epsilon) + n^i_4(n, \epsilon).
\end{equation}
The proof proceeds by bounding, in expectation, each of the four terms.

Observe that, by the structure of the index function $u_i$, 
\begin{equation}
\begin{split}
n^i_1(n, \epsilon) & \leq \sum_{t = 3N}^n \mathbbm{1} \left\{ \pi(t+1) = i, ( \mu_i + \tilde{\epsilon} ) + \sigma_i \sqrt{1 + \epsilon} \sqrt{ t^{ \frac{2}{T^i_\pi(t) - 2} } - 1 } \geq \mu^* - \tilde{\epsilon} \right\} \\
& = \sum_{t = 3N}^n \mathbbm{1} \left\{ \pi(t+1) = i, T^i_\pi(t) \leq \frac{ 2 \ln t }{   \ln \left( 1 + \frac{ 1 }{ \sigma_i^2 } \frac{ \left( \Delta_i - 2 \tilde{\epsilon} \right)^2 }{ \left(1 + \epsilon \right) } \right) } + 2\right\} \\
& = \sum_{t = 3N}^n \mathbbm{1} \left\{ \pi(t+1) = i, T^i_\pi(t) \leq \frac{ 2 \ln t }{   \ln \left( 1 + \frac{ \Delta_i^2 }{ \sigma_i^2 } \frac{ \left( 1 - \epsilon \right)^2 }{ \left(1 + \epsilon \right) } \right) } + 2\right\} \\
& \leq \sum_{t = 3N}^n \mathbbm{1} \left\{ \pi(t+1) = i, T^i_\pi(t) \leq \frac{ 2 \ln n }{   \ln \left( 1 + \frac{ \Delta_i^2 }{ \sigma_i^2 } \frac{ \left( 1 - \epsilon \right)^2 }{ \left(1 + \epsilon \right) } \right) } + 2\right\} \\
& \leq \sum_{t = 1}^n \mathbbm{1} \left\{ \pi(t+1) = i, T^i_\pi(t) \leq \frac{ 2 \ln n }{   \ln \left( 1 + \frac{ \Delta_i^2 }{ \sigma_i^2 } \frac{ \left( 1 - \epsilon \right)^2 }{ \left(1 + \epsilon \right) } \right) } + 2\right\} \\
& \leq \frac{ 2 \ln n }{   \ln \left( 1 + \frac{ \Delta_i^2 }{ \sigma_i^2 } \frac{ \left( 1 - \epsilon \right)^2 }{ \left(1 + \epsilon \right) } \right) } + 2 + 2.
\end{split}
\end{equation}
The last inequality follows, observing that $T^i_\pi(t)$ may be expressed as the sum of $\pi(t) = i$ indicators, and seeing that the additional condition bounds the number of non-zero terms in the above sum. The additional $+2$ simply accounts for the $\pi(1) = i$ term and the $\pi(n+1) = i$ term. 
 Note, this bound is sample-path-wise.


For the second term,
\begin{equation}\label{en:second-term}
\begin{split}
n^i_2(n, \epsilon) & \leq \sum_{t = 3N}^n \mathbbm{1} \{ \pi(t+1) = i, S^2_i(T^i_\pi(t)) > \sigma_i^2(1 + \epsilon) \} \\
& =  \sum_{t = 3N}^n \sum_{k = 2}^t \mathbbm{1} \{ \pi(t+1) = i, S^2_i(k) > \sigma_i^2(1 + \epsilon), T^i_\pi(t) = k \} \\
& =  \sum_{t = 3N}^n \sum_{k = 2}^t \mathbbm{1} \{ \pi(t+1) = i, T^i_\pi(t) = k \}\mathbbm{1} \{ S^2_i(k) > \sigma_i^2(1 + \epsilon) \} \\
& \leq \sum_{k = 2}^n \mathbbm{1} \{ S^2_i(k) > \sigma_i^2(1 + \epsilon) \} \sum_{t = k}^n \mathbbm{1} \{ \pi(t+1) = i, T^i_\pi(t) = k \} \\
& \leq \sum_{k = 2}^n \mathbbm{1} \{ S^2_i(k) > \sigma_i^2(1 + \epsilon) \}.
\end{split}
\end{equation}
The last inequality follows as, for fixed $k$, $\{ \pi(t+1) = i, T^i_\pi(t) = k \}$ may be true for at most one value of $t$. Recall that $k S^2_i(k) / \sigma_i^2$ has the distribution of a $\chi_{k-1}^2$ random variable. Letting $U_{k} \sim \chi^2_{k}$, from the above we have
\begin{equation}
\begin{split}
\mathbb{E} \left[ n^i_2(n, \epsilon) \right] & \leq \sum_{k = 2}^n \mathbb{P} \left( S^2_i(k) > \sigma_i^2(1 + \epsilon) \right) \\
& \leq \sum_{k = 2}^\infty \mathbb{P} \left( U_{k-1} / k > (1 + \epsilon) \right) \\
& \leq \sum_{k = 2}^\infty \mathbb{P} \left( U_{k-1}  / (k-1) > (1 + \epsilon) \right) \\
& = \sum_{k = 1}^\infty \mathbb{P} \left( U_{k} > k (1 + \epsilon) \right) \\ 
& \leq \frac{1}{ \sqrt{ \frac{ e^\epsilon }{ 1 + \epsilon } } - 1} \leq \frac{8}{\epsilon^2} < \infty.
\end{split}
\end{equation}
The penultimate step is a Chernoff bound on the terms, $\mathbb{P} \left( U_{k} > k (1 + \epsilon) \right) \leq (e^{-\epsilon} ( 1 + \epsilon ))^{k/2}.$

To bound the third term, a similar rearrangement to Eq. \eqref{en:second-term} (using the sample mean instead of the sample variance) yields:
\begin{equation}
\begin{split}
n^i_3(n, \epsilon) & \leq \sum_{t = 3N}^n \mathbbm{1} \{ \pi(t+1) = i, \bar{X}^i_{ T^i_\pi(t) } > \mu_i + \tilde{\epsilon} \}  \leq \sum_{k = 2}^n \mathbbm{1} \{  \bar{X}^i_{ k } > \mu_i + \tilde{\epsilon} \}.
\end{split}
\end{equation}
Recalling that $\bar{X}^i_k - \mu_i \sim Z \sigma_i / \sqrt{k}$ for $Z$ a standard normal,
\begin{equation}
\begin{split}
\mathbb{E} \left[ n^i_3(n, \epsilon) \right] & \leq \sum_{k = 2}^n \mathbb{P} \left(  \bar{X}^i_{ k } > \mu_i + \tilde{\epsilon} \right) \leq \sum_{k = 1}^\infty \mathbb{P} \left(   Z \sigma_i / \sqrt{k} >  \tilde{\epsilon} \right)  \leq \frac{1}{ e^{  \frac{ \tilde{\epsilon}^2 }{2 \sigma_i^2 } } - 1} \leq \frac{ 2 \sigma_i^2  }{ \tilde{\epsilon}^2 } < \infty.
\end{split}
\end{equation}
The penultimate step is a Chernoff bound on the terms, $\mathbb{P} \left(   Z >  \delta \sqrt{k} \right) \leq e^{- k \delta^2 / 2 }$.

To bound the $n^i_4$ term, observe that in the event $\pi(t+1) = i$, from the structure of the policy it must be true that $u_i(t, T^i_\pi(t)) = \max_j u_j(t, T^j_\pi(t))$. Thus, if $i^*$ is some bandit such that $\mu_{i^*} = \mu^*$, $u_{i^*}(t, T^{i^*}_\pi(t) ) \leq u_i(t, T^i_\pi(t))$. In particular, we take $i^*$ to be a bandit that not only achieves the maximal mean $\mu^*$, but also the minimal variance among optimal bandits, $\sigma^2_{i^*} = \sigma^2_*$. We have the following bound,
\begin{equation}
\begin{split}
n^i_4(n, \epsilon) & \leq \sum_{t = 3N}^n \mathbbm{1} \{ \pi(t+1) = i, u_{i^*}(t, T^{i^*}_\pi(t)) < \mu^* - \tilde{ \epsilon } \} \\
& \leq \sum_{t = 3N}^n \mathbbm{1} \{ u_{i^*}(t, T^{i^*}_\pi(t)) < \mu^* - \tilde{ \epsilon } \} \\
& \leq \sum_{t = 3N}^n \mathbbm{1} \{ u_{i^*}(t, s) < \mu^* - \tilde{ \epsilon } \text{ for some } 3 \leq s \leq t\}.
\end{split}
\end{equation}
The last step follows as for $t$ in this range, $3 \leq T^{i^*}_\pi(t) \leq t$. Hence
\begin{equation}\label{eqn:n-4-bound}
\mathbb{E} \left[ n^i_4(n, \epsilon) \right] \leq \sum_{t = 3N}^n \mathbb{P} \left( u_{i^*}(t, s) < \mu^* - \tilde{ \epsilon } \text{ for some } 3 \leq s \leq t \right).
\end{equation}
As an aside, this is essentially the point at which the conjectured Eq. \eqref{eqn:prob-conjecture} would have come into play for the proof of the optimality of $\pi_{\text{BK}}$, bounding the growth of the corresponding term for that policy. We will essentially prove a successful version of that conjecture here. Define the events $A^*_{s, t, \epsilon} = \{ u_{i^*}(t, s) < \mu^* - \tilde{ \epsilon } \}$. Observing the distributions of the sample mean and sample variance, we have (similar to Eq. \eqref{eqn:prop-restatement}) for $Z$ a standard normal and $U_{s-1} \sim \chi^2_{s-1}$, with $U,\ Z$ independent,
\begin{equation}
\begin{split}
\mathbb{P} \left( A^*_{s, t, \epsilon} \right) & = \mathbb{P} \left( \frac{ \tilde{ \epsilon } }{ \sigma_{*} } \sqrt{s} + \sqrt{U_{s-1}} \sqrt{t^{\frac{2}{s-2}} - 1} < Z \right) \\
& \leq \frac{ e^{-{(\tilde{ \epsilon }/\sigma_*)}^2 s/2} (s-2) }{ 2 {(\tilde{ \epsilon }/\sigma_*)}^2 s \sqrt{e (s-1)}  } \left( \frac{ t^{ -1 } }{ \ln t } \right) \\
& \leq \frac{ e^{-{( \tilde{ \epsilon }/\sigma_*)}^2 s/2} }{ 2 {(\tilde{ \epsilon }/\sigma_*)}^2 } \frac{1}{ \sqrt{ e s } } \left( \frac{ t^{ -1 } }{ \ln t } \right) \\
& \leq \left(  \frac{1}{  2 {(\tilde{ \epsilon }/\sigma_*)}^2 \sqrt{e} } \right) \frac{ e^{-{(\tilde{ \epsilon }/\sigma_*)}^2 s/2} }{ \sqrt{ s } } \left( \frac{ t^{ -1 } }{ \ln t } \right). \\
\end{split}
\end{equation}
where the first inequality follows as an application of Proposition \ref{prop:probability-bound}, and the second since $s \geq 3$. Applying a union bound to Eq. \eqref{eqn:n-4-bound},
\begin{equation}
\begin{split}
\mathbb{E} \left[ n^i_4(n, \epsilon) \right] & \leq \sum_{t = 3N}^n \sum_{s = 3}^t    \mathbb{P} \left( A^*_{s, t, \epsilon} \right) \\
& \leq  \sum_{t = 3N}^n \sum_{s = 3}^t \left(  \frac{1}{  2 {(\tilde{ \epsilon }/\sigma_*)}^2 \sqrt{e} } \right) \frac{ e^{-{(\tilde{ \epsilon }/\sigma_*)}^2 s/2} }{ \sqrt{ s } } \left( \frac{ t^{ -1 } }{ \ln t } \right) \\
& \leq  \left(  \frac{1}{  2 {(\tilde{ \epsilon }/\sigma_*)}^2 \sqrt{e} } \right) \int_{s = 0}^\infty \frac{ e^{-{(\tilde{ \epsilon }/\sigma_*)}^2 s/2} }{ \sqrt{ s } } ds  \int_{t = e}^n \left( \frac{ t^{ -1 } }{ \ln t } \right) dt  \\
& =  \left(  \frac{1}{  2 {(\tilde{ \epsilon }/\sigma_*)}^2 \sqrt{e} } \right) \frac{ \sqrt{2\pi} }{ (\tilde{ \epsilon }/\sigma_*) } \ln \ln n \\
& =   \sqrt{ \frac{ \pi }{ 2 e } } \frac{ \sigma_*^3 }{ \tilde{ \epsilon }^3 } \ln \ln n . \\
\end{split}
\end{equation}
The bounds follow, removing the dependence of the $s$-sum on $t$ by extending it to $\infty$, and bounding the sums by integrals of the (decreasing) summands by slightly extending the range of each. From the above results, and observing that $T^i_\pi(n) \leq T^i_\pi(n+1)$, it follows from Eq. \eqref{eqn:t-n-relation} that for any $\epsilon$ such that $0 < \epsilon < 1$,
\begin{equation}
\begin{split}
\mathbb{E} \left[ T^i_\pi(n) \right] & \leq \frac{ 2 \ln n }{   \ln \left( 1 + \frac{ \Delta_i^2 }{ \sigma_i^2 } \frac{ \left( 1 - \epsilon \right)^2 }{ \left(1 + \epsilon \right) } \right) } + 4 + \frac{8}{ \epsilon^2} + \frac{ 2 \sigma_i^2  }{ \tilde{\epsilon}^2 } + \sqrt{ \frac{ \pi }{ 2 e } } \frac{ \sigma_i^3 }{ \tilde{ \epsilon }^3 } \ln \ln n \\
& \leq \frac{ 2 \ln n }{   \ln \left( 1 + \frac{ \Delta_i^2 }{ \sigma_i^2 } \frac{ \left( 1 - \epsilon \right)^2 }{ \left(1 + \epsilon \right) } \right) } + 4 + \frac{8}{ \epsilon^2} + \frac{ 8 \sigma_i^2  }{ \Delta_i^2 \epsilon^2 } + \sqrt{ \frac{ \pi }{ 2 e } } \frac{ 8 \sigma_*^3 }{ \Delta_i^3 \epsilon^3 } \ln \ln n.
\end{split}
\end{equation}
The result then follows from the definition of regret in Eq. \eqref{eqn:regret}.
%
%

\end{proof}

{\bf Remark 2}{ Numerical Regret Comparison:}
Figure 1 shows the results of a small simulation study done on a set of six populations with means and variances given in  Table 1. It provides plots of the regrets when implementing policies $\cpi, \api$, and $\pi_G$ a `greedy' policy that always activates the bandit with the current highest average.  
 Each policy was implemented over a horizon of 100,000 activations, each replicated 10,000 times to produce a good estimate of the average regret $R_\pi(n)$ over the times indicated. 
 The left plot is on the  time scale of the first  $10,000 $ activations, and the right
  is  on   the  full time scale of $100,000 $ activations.

\begin{center}
{ \small
 \begin{tabular}{|l||r|r|r|r|r|r|}
  \hline
   $\mu_i $ & 8 &8 &7.9 & 7 & -1 & 0\\
       \hline
$\sigma^2_i$  &  1 &1.4 & 0.5 & 3 & 1 & 4\\
  \hline  \hline    
  \multicolumn{7}{l}{Table  1}    \\
\end{tabular}
}
 \end{center}

 \begin{figure}[h!]
\begin{center}
\includegraphics[width=1\textwidth]{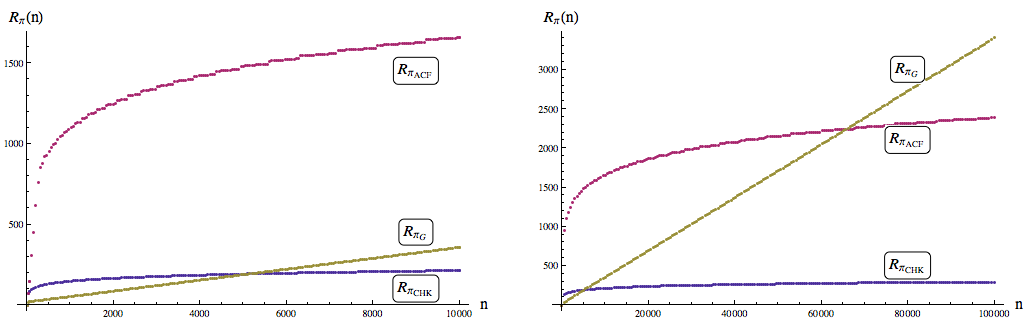}
\caption{\small Numerical Regret Comparison of \api, \cpi, and $\pi_{\text{G}}$; 
Left: $[0,10,000]$ range, Right:  $[0,100,000]$ range.} 
\end{center}
\end{figure}

{\bf Remark 3}{ Bounds and Limits:}
Figure 2 shows first (left) a comparison of the theoretical bounds on the regret, $B_{\api}(n)$ and $B_{\cpi}(n)$ representing the theoretical regret bounds of the RHS  of Eq. \eqref{en:a-bound} and Eq. \eqref{en:finite-bound} respectively, taking $\epsilon = (\ln n)^{1/4}$ in the latter case, for the means and variances indicated in  Table 1. Additionally, Figure 2 (right) shows the convergence of $R_{\cpi}(n) / \ln n$ to the theoretical lower bound $\mathbb{M}_{\text{BK}}(\underline{\mu}, \underline{\sigma}^2)$.

 \begin{figure}[h!]
\begin{center}
\includegraphics[width=.5\textwidth]{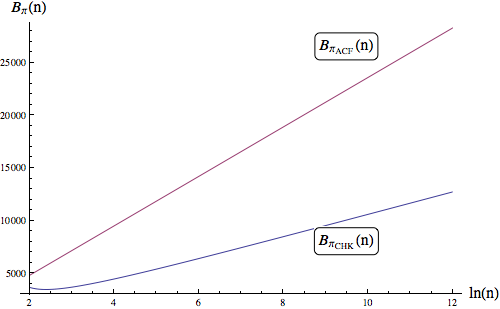}\includegraphics[width=.5\textwidth]{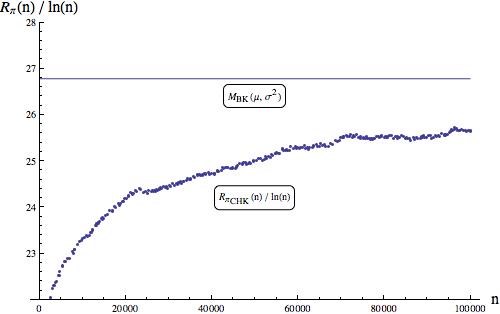}
\caption{\small Left: Plots of $B_{\api}(n)$ and $B_{\cpi}(n)$. \  \ Right: Convergence of $R_{\cpi} (n)/\ln (n) $ to $\mathbb{M}_{\text{BK}}(\underline{\mu}, \underline{\sigma}^2) $
} 
\end{center}
\end{figure}

\section{A Comparison of \cpi and Thompson Sampling}\label{sec:comparison}
\cite{honda13} proved that for $\alpha < 0$, the following Thompson sampling algorithm is asymptotically optimal, i.e., 
$\lim_{n\to\infty}   R_{\pi_{\text{CHK}}}(n) / \ln n  = \mathbb{M}_{\text{BK}}(\underline{\mu}, \underline{\sigma}^2).$
\begin{tcolorbox}[colback=blue!1, arc=3pt, width=.94\linewidth]
\textbf{Policy     $ \boldmath \pi_{\text{TS}} $  (TS-NORMAL$\!\,^\alpha$)}
\begin{itemize}
\item[i)] Initially, sample each bandit $\tilde{n} \geq \max(2, 3-\lfloor 2\alpha \rfloor)$ times.
\item[ii)] For $n \geq \tilde{n}\,$: For each $i$ generate a random sample $U^i_n$ from a posterior distribution for $\mu_i$, given $\left( \bar{X}^i_{T^i_\pi(n)}, S^2_i( T^i_\pi(n) )\right)$, and a prior for $\left( \mu_i, \sigma^2_i \right) \propto \left( \sigma^2_i \right)^{-1-\alpha}$. 
\item[iii)]  Then, take 
\begin{equation}
\pi_{\text{TS}}(n+1) = \argmax_i\ U^i_n.
\end{equation}
\end{itemize}
\end{tcolorbox}

Policies $\pi_{\text{TS}}$ and \cpi\ differ decidedly in structure. One key difference, $\pi_{\text{TS}}$ is an inherently randomized policy, while decisions under \cpi\ are completely determined given the bandit results at a given time. Given that both $\pi_{\text{TS}}$ and \cpi\ are asymptotically optimal, it is interesting to compare the performances of these two algorithms over finite time horizons, and observe any practical differences between them. To that end, two small simulation studies were done for different sets of bandit parameters $(\underline{\mu}, \underline{\sigma}^2)$. In each case, 
the uniform prior  
 $\alpha = -1$ was used. The simulations were carried out on a 10,000 round time horizon, and replicated sufficiently many times to get good estimates for the expected regret over the times indicated.
\begin{figure}[h!]
\begin{center}
\includegraphics[width=1\textwidth]{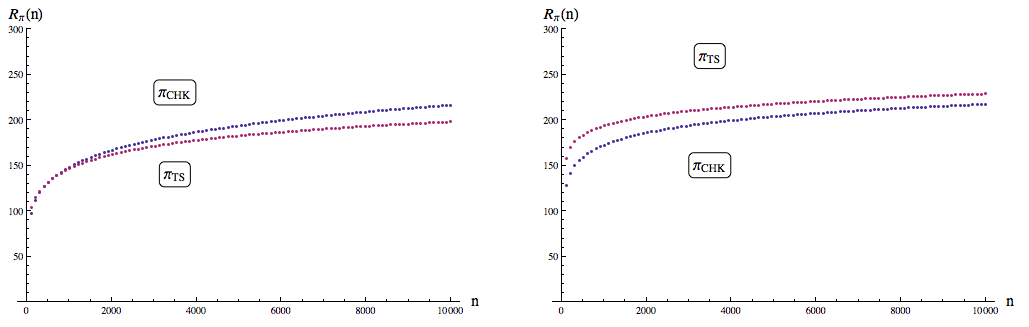}
\caption{\small Numerical Regret Comparison of \cpi and $\pi_{\text{TS}}$ for the 
 parameters, of Table 1, left and Table 2, right.} 
\end{center}
\end{figure}
 \begin{center}
{ \small
 \begin{tabular}{|l||r|r|r|r|r|r|}
  \hline
   $\mu_i $ & 10 & 9 & 8 & 7 & -1 & 0\\
       \hline
$\sigma^2_i$  &  8 &1 & 1 & 0.5 & 1 & 4\\
  \hline  \hline    
  \multicolumn{7}{l}{Table  2}    \\
\end{tabular}
}
 \end{center}

\begin{figure}[h!]
\begin{center}
\includegraphics[width=1\textwidth]{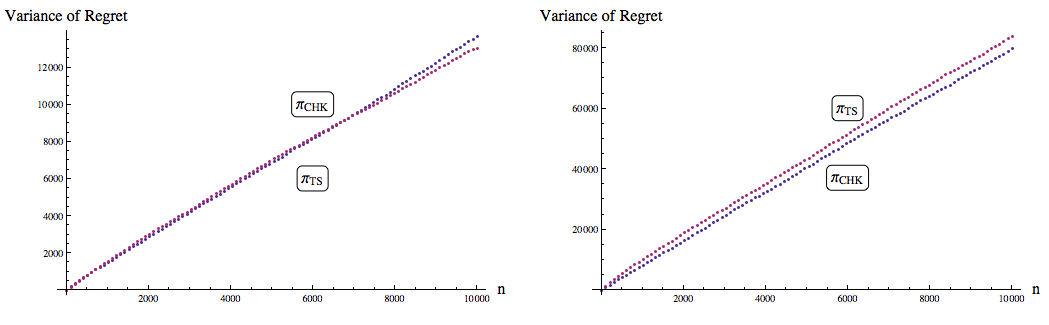}
\caption{\small Numerical comparison of variance of sample regret for \cpi and $\pi_{\text{TS}}$ for different parameters, of Table 1, left and Table 2, right.} 
\end{center}
\end{figure}

We observe from the above, and from general sampling of bandit parameters, that $\pi_{\text{TS}}$ and \cpi\ generally produce comparable expected regret. A general exploration of random parameters suggests that, on average, $\pi_{\text{TS}}$ is slightly superior to \cpi\ in cases where all bandits have roughly equal variances, while \cpi\ has an edge when the optimal bandits have large variance relative to the other bandits, and the size of the bandit discrepancies. It is additionally interesting to note that in the cases pictured above, the superior policy also demonstrated the smaller variance in sample regret (Figure 4). Additional numerical experiments, not pictured here, indicate that the superior policy in each case may exhibit a slightly heavier tail distribution towards larger regret. In general, the question of which policy is superior is largely context specific.

\bibliography{mab2015}

 \ \\

{\bf Acknowledgement:} 
We gratefully  acknowledge support for this project
from the National Science Foundation (NSF grant CMMI-14-50743). 

\appendix
\section{Additional Proofs}
\begin{proof}[of Proposition \ref{prop:probability-bound}]
Let $P = \mathbb{P}\left( \delta + \sqrt{U} \sqrt{ k^{2/p} - 1} < Z \right)$. Note immediately, $P \geq \mathbb{P}\left( \delta + \sqrt{U} k^{1/p} < Z \right)$. Further,
\begin{equation}
\begin{split}
P & \geq \mathbb{P}\left( \delta + \sqrt{U} k^{1/p} < Z \text{ and } \sqrt{U} k^{1/p} \geq \delta \right) \\
& \geq \mathbb{P}\left( 2\sqrt{U} k^{1/p} < Z \text{ and } \sqrt{U} k^{1/p} \geq \delta \right) \\
& = \int_{ \frac{\delta^2}{ k^{2/p} }  }^\infty \int_{ 2 \sqrt{ u } k^{1/p} }^\infty \frac{e^{-z^2/2}}{\sqrt{2\pi}} f_d(u) dz du.
\end{split}
\end{equation}
Where $f_d(u)$ is taken to be the density of a $\chi_d^2$-random variable. Letting $\tilde{u} = k^{2/p}u$,
\begin{equation}
\begin{split}
P &  \geq \frac{1}{k^{2/p}} \int_{ \delta^2 }^\infty \int_{ 2 \sqrt{ \tilde{u} } }^\infty \frac{e^{-z^2/2}}{\sqrt{2\pi}} f_d\left(\frac{\tilde{u}}{k^{2/p}}\right) dz d\tilde{u} \\
& =  \frac{1}{k^{2/p}} \int_{ \delta^2  }^\infty \int_{ 2 \sqrt{ \tilde{u} } }^\infty \frac{e^{-z^2/2}}{\sqrt{2\pi}} \frac{1}{2^{d/2} \Gamma(d/2) } { \left( \frac{\tilde{u}}{ k^{2/p} } \right) }^{d/2 - 1} e^{-\frac{\tilde{u}}{ 2 k^{2/p} }} dz d\tilde{u} \\
& =  \left( \frac{1}{k^{2/p}}  \right)^{d/2} \int_{ \delta^2  }^\infty \int_{ 2 \sqrt{ \tilde{u} } }^\infty \frac{e^{-z^2/2}}{\sqrt{2\pi}} \frac{1}{2^{d/2} \Gamma(d/2) } \tilde{u}^{d/2 - 1} e^{-\frac{\tilde{u}}{ 2 k^{2/p} }} dz d\tilde{u}
\end{split}
\end{equation}
Observing that $k^{2/p} \geq 1$,
\begin{equation}
\begin{split}
P & \geq  \left( \frac{1}{k^{2/p}}  \right)^{d/2} \int_{ \delta^2 }^\infty \int_{ 2 \sqrt{ \tilde{u} } }^\infty \frac{e^{-z^2/2}}{\sqrt{2\pi}} \frac{1}{2^{d/2} \Gamma(d/2) } \tilde{u}^{d/2 - 1} e^{-\frac{\tilde{u}}{ 2}} dz d\tilde{u}\\
& = k^{-d/p} \mathbb{P} \left( 2 \sqrt{U} \leq Z \text{ and } U \geq \delta^2 \right) \\
& = \frac{1}{2} k^{-d/p} \mathbb{P} \left( 4 U \leq Z^2 \text{ and } U \geq \delta^2 \right) = \frac{1}{2} k^{-d/p} \mathbb{P} \left( \frac{1}{4} Z^2 \geq U \geq \delta^2 \right) \\
\end{split}
\end{equation}

The exchange from integral to probability is simply the interpretation of the integrand as the joint pdf of $U$ and $Z$.

For the upper bound, we utilize the classic normal tail bound, $\mathbb{P}\left( x < Z \right) \leq e^{-x^2/2}/(x \sqrt{2\pi}).$
\begin{equation}
\begin{split}
P & \leq \mathbb{E} \left[ \frac{e^{-\left(\delta + \sqrt{U} \sqrt{ k^{2/p} - 1}\right)^2/2}}{ (\delta + \sqrt{U} \sqrt{ k^{2/p} - 1}) \sqrt{2\pi} } \right]  \leq \frac{ e^{-\delta^2/2} }{ \delta \sqrt{2\pi} } \mathbb{E} \left[ e^{- \delta \sqrt{U} \sqrt{ k^{2/p}-1 } - \frac{1}{2}U \left( k^{2/p}-1\right)}\right].
\end{split}
\end{equation}
Observing the bound that for positive $x$, $e^{-x} \leq 1/x$, and recalling that $d \geq 2$,
\begin{equation}\label{eqn:37}
\begin{split}
P & \leq \frac{ e^{-\delta^2/2} }{ \delta \sqrt{2\pi} } \mathbb{E} \left[ \frac{ e^{ - \frac{1}{2}U \left( k^{2/p}-1\right)} }{\delta \sqrt{U} \sqrt{ k^{2/p}-1 }}\right] \\
& = \frac{ e^{-\delta^2/2} }{ \delta^2 \sqrt{2\pi} \sqrt{ k^{2/p}-1 } } \mathbb{E} \left[ U^{ - \frac{1}{2} } e^{ - \frac{1}{2}U \left( k^{2/p}-1\right)} \right] \\
& = \frac{ e^{-\delta^2/2} }{ \delta^2 \sqrt{2\pi} \sqrt{ k^{2/p}-1 } } \left( \frac{k^{ (1-d)/p } \Gamma \left( \frac{d}{2} - \frac{1}{2} \right) }{ \sqrt{2} \Gamma \left( \frac{d}{2} \right) }\right).
\end{split}
\end{equation}
Here we utilize the following bounds: $e^x - 1 \geq (e/2)x^2$, which is easy to prove, and $\Gamma( d/2 - 1/2 ) / \Gamma( d/2 ) \leq \sqrt{2 \pi / d }$, which may be proved on integer $d \geq 2$ by induction. This yields:
\begin{equation}
P \leq \frac{ e^{-(1+\delta^2)/2} p }{ 2\delta^2  \ln k  } \frac{ k^{ (1-d)/p } }{ \sqrt{d} }.
\end{equation}
This completes the proof.

{\bf Remark 4}{ Room for Improvement:}
The choice of the $e^x - 1 \geq (e/2)x^2$ bound above was in fact arbitrary - other bounds, such as involving alternative powers of $x$, could be used. This would influence how the resulting bound on $P$ is utilized, for instance in the proof of Theorem \ref{thm:finite-time}. The use of $e^{-x} \leq 1/x$ in Eq. \eqref{eqn:37} should be considered similarly.
\end{proof}

\begin{proposition}\label{prop:conjecture-false}
Conjecture 1 is {\bf false} and for each $i$, for $\epsilon > 0$,
\begin{equation}\label{eqn:very-bad}
\frac{ \mathbb{P} \left( \bar{X}^i_{ j } + S_i( j ) \sqrt{ k^{2/j} - 1} < \mu_i - \epsilon \text{ for some } 2 \leq j \leq k \right) }{ 1/k } \to \infty \text{ as } k \to \infty.
\end{equation}.
\end{proposition}
\begin{proof}[of Proposition \ref{prop:conjecture-false}]
Define the events $A^i_{j, k, \epsilon} = \{ \bar{X}^i_{ j } + S_i( j ) \sqrt{ k^{2/j} - 1} < \mu_i - \epsilon \}$. As the samples are taken to be normally distributed with mean $\mu_i$ and variance $\sigma_i^2$, we have that $\bar{X}^i_j - \mu_i \sim Z \sigma_i/\sqrt{j}$ and $S_i^2(j) \sim \sigma_i^2 U / j$, where $Z$ is a standard normal, $U \sim \chi^2_{j-1}$, and $Z, U$ independent. Hence,
\begin{equation}\label{eqn:prop-restatement}
\begin{split}
\mathbb{P} ( A^i_{j, k, \epsilon} ) & = \mathbb{P} \left( Z \frac{\sigma_i}{\sqrt{j}} + \sqrt{ U \frac{ \sigma_i^2}{j} } \sqrt{ k^{2/j} - 1} < - \epsilon \right)  = \mathbb{P} \left( \frac{ \epsilon }{ \sigma_i } \sqrt{j} + \sqrt{ U } \sqrt{ k^{2/j} - 1} < Z  \right).
\end{split}
\end{equation}
The last step is simply a re-arrangement, and an observation on the symmetry of the distribution of $Z$. For $j \geq 3$, we may apply Proposition \ref{prop:probability-bound} here for $d = j-1$, $p = j$, to yield
\begin{equation}\label{eqn:pi-lower-bound}
\mathbb{P} ( A^i_{j, k, \epsilon} ) \geq \frac{1}{2} \frac{ k^{1/j} }{k} \mathbb{P} \left( \frac{1}{4} Z^2 \geq U\geq \frac{ \epsilon^2 }{ \sigma_i^2} j \right).
\end{equation}
For a fixed $j_0 \geq 3$, for $k \geq j_0$ we have
\begin{equation}
\mathbb{P} \left( A^i_{j, k, \epsilon} \text{ for some } 2 \leq j \leq k \right) \geq \mathbb{P} ( A^i_{j_0, k, \epsilon} ) \geq O\left( 1/k \right) k^{1/j_0}.
\end{equation}
The proposition follows immediately. 
\end{proof}

\begin{proposition}\label{prop:that-d-inequality}
For $G > 0$, $0 \leq \epsilon < 1/2$, the following holds:
\begin{equation}
\frac{1}{\ln \left( 1 + G \frac{(1-\epsilon)^2}{1 + \epsilon} \right) } \leq \frac{1}{\ln \left( 1 + G \right) } + \frac{10 G}{ \left(1 + G \right) \ln\left(1 + G \right)^2 } \epsilon.
\end{equation}
\end{proposition}
\begin{proof}
For any $G > 0$, the function $1/\ln\left(1 + G \frac{(1-\epsilon)^2}{1 + \epsilon}\right)$ is positive, increasing, and convex on $\epsilon \in [0, 1)$ (Proposition \ref{prop:convexity}). For a given $G > 0$, noting that the above inequality holds (as equality) at $\epsilon = 0$, due to the convexity it suffices to show that the inequality is satisfied at $\epsilon = 1/2$, or
\begin{equation}
\frac{1}{ \ln \left(1 + \frac{G}{6}\right) } \leq \frac{5 G}{(1 + G) \ln \left(1 + G\right)^2 } + \frac{1}{ \ln (1 + G) }.
\end{equation}
Equivalently, we consider the inequality
\begin{equation}\label{en:new-bound}
0 \leq \frac{5 G}{(1 + G) } + \ln (1 + G) - \frac{ \ln \left(1 + G\right)^2 }{ \ln \left(1 + \frac{G}{6}\right) }.
\end{equation}
Define the function $F(G)$ to be the RHS of Ineq. \eqref{en:new-bound}. Note that as $G \to 0$, $F(G) \to 0$, and in simplified form we have (for $G > 0$ and the limit as $G \to 0$),
\begin{equation}
F'(G) = \frac{ \left( (1 + G)\ln(1 + G) - (6 + G) \ln\left(1 + \frac{G}{6}\right)\right)^2}{(1+G)^2 (6 + G) \ln \left(1 + \frac{G}{6}\right)^2} \geq 0.
\end{equation}
It follows that $F(G) \geq 0$, and hence the desired inequality holds at $\epsilon = 1/2$. This completes the proof.
\end{proof}
\begin{proposition}\label{prop:convexity}
The function $H_G(\epsilon) = 1/\ln\left(1 + G \frac{(1-\epsilon)^2}{1 + \epsilon}\right)$ is positive, increasing, and convex in $\epsilon \in [0, 1),$ for any constant  $G > 0.$ 
\end{proposition}
\begin{proof}
 That $H_G(\epsilon)$ is positive and increasing in $\epsilon$, follows immediately from inspection of $H_G$ and $H'_G$, given the hypotheses on $G, $ and $\epsilon$. 
To demonstrate convexity, by inspection of the terms of $H''_G(\epsilon)$, it suffices to show that for all relevant $G,$ and $ \epsilon$, the following inequality holds.
\begin{equation}
2G(1-\epsilon)^2(3 + \epsilon)^2 + \left( -8(1 + \epsilon) + G(1-\epsilon)^2(1 + \epsilon(6 + \epsilon))\right) \ln \left(1 + G \frac{(1-\epsilon)^2}{1 + \epsilon} \right) \geq 0.
\end{equation}
Defining $C= G (1-\epsilon)^2/(1 + \epsilon)$, it is sufficient to show that for all $C > 0$ and $\epsilon \in [0, 1)$ (eliminating a factor of $(1 + \epsilon)$ from the above),
\begin{equation}
2C(3 + \epsilon)^2 + \left( -8 + C(1 + \epsilon(6 + \epsilon))\right) \ln \left(1 + C \right) \geq 0.
\end{equation}
Defining $J_C(\epsilon)$ as the LHS of the above, note that $J'_C(\epsilon) = 2C(3 + \epsilon)(2 + \ln(1 + C)) > 0$. It suffices then to show $J_C(0) \geq 0$, or $18C + (C-8)\ln(1 + C) \geq 0$. Note this holds at $C = 0$, and $d/dC[ J_C(0) ] = (10 + 19C)/(1 + C) + \ln(1 + C) > 0$ for $C \geq 0$. Hence, $J_C(\epsilon) \geq 0$, and $H''_G(\epsilon) \geq 0$.
\end{proof}

\end{document}